\documentclass{article}

\usepackage{neurips_2024}

\usepackage{graphicx}
\usepackage{wrapfig}
\usepackage[utf8]{inputenc} % allow utf-8 input
\usepackage[T1]{fontenc}    % use 8-bit T1 fonts
\usepackage{hyperref}       % hyperlinks
\usepackage{url}            % simple URL typesetting
\usepackage{booktabs}       % professional-quality tables
\usepackage{amsfonts}       % blackboard math symbols
\usepackage{nicefrac}       % compact symbols for 1/2, etc.
\usepackage{microtype}      % microtypography
\usepackage{xcolor}         % colors

\usepackage{amsthm}

\newtheorem{theorem}{Theorem}

\usepackage[english]{babel}

\usepackage{microtype}
\usepackage{graphicx}
\usepackage{subfigure}
\usepackage{booktabs} % for professional tables
\usepackage{amssymb}
\usepackage{multicol}
\usepackage{amsmath}
\usepackage{graphicx}
\usepackage{subfigure}
\usepackage{times}
\usepackage{floatrow}
\usepackage{subcaption}
\usepackage{latexsym}
\usepackage{colortbl}
\usepackage{multirow}
\usepackage[T1]{fontenc}
\usepackage[utf8]{inputenc}

\usepackage{microtype}

\usepackage{inconsolata}

\title{TernaryLLM: Ternarized Large Language Model}

\author{
Tianqi Chen$^{1,2}$,
Zhe Li$^{3}$,
Weixiang Xu$^1$,
Zeyu Zhu$^{1}$, \\
\textbf{Dong Li}$^{3}$, 
\textbf{Lu Tian}$^3$,
\textbf{Emad Barsoum}$^3$, \\ 
\textbf{Peisong Wang}$^{1,2,4}$,
\textbf{Jian Cheng}$^{* 1,2,4,5}$ \\
$^1$Institute of Automation, Chinese Academy of Sciences \\
$^2$School of Artificial Intelligence, University of Chinese Academy of Sciences\\ 
$^3$Advanced Micro Devices, Inc., Beijing, China \\
$^4$AIRIA  $^5$Maicro.ai  \\
\texttt{\{chentianqi2023,xuweixiang2018,zhuzeyu2021\}@ia.ac.cn}, \\
\texttt{\{Z.Li,d.li,lu.tian,Emad.Barsoum\}@amd.com}, \\ 
\texttt{\{peisong.wang,jcheng\}@nlpr.ia.ac.cn}
}

\begin{document}

\maketitle
\begin{abstract}
Large language models (LLMs) have achieved remarkable performance on Natural Language Processing (NLP) tasks, but they are hindered by high computational costs and memory requirements. Ternarization, an extreme form of quantization, offers a solution by reducing memory usage and enabling energy-efficient floating-point additions. However, applying ternarization to LLMs faces challenges stemming from outliers in both weights and activations.
In this work, observing asymmetric outliers and non-zero means in weights, we introduce Dual Learnable Ternarization (DLT), which enables both scales and shifts to be learnable. We also propose Outlier-Friendly Feature Knowledge Distillation (OFF) to recover the information lost in extremely low-bit quantization. The proposed OFF can incorporate semantic information and is insensitive to outliers. At the core of OFF is maximizing the mutual information between features in ternarized and floating-point models using cosine similarity. Extensive experiments demonstrate that our TernaryLLM surpasses previous low-bit quantization methods on the standard text generation and zero-shot benchmarks for different LLM families. Specifically, for one of the most powerful open-source models, LLaMA-3, our approach (W1.58A16) outperforms the previous state-of-the-art method (W2A16) by 5.8 in terms of perplexity on C4 and by 8.2\% in terms of average accuracy on zero-shot tasks.
\end{abstract}

\section{Introduction}
\label{introduction}
Large language models (LLMs) \cite{LLaMa, OPT} have demonstrated impressive performance across various language tasks. Although LLMs' representative ability improves as the parameters scale exponentially, the enormous parameters pose significant challenges on memory footprint and low latency inference \cite{AWQ, Smoothquant}. Therefore, resolving the above problems for the practical deployment of LLMs is an emergency. 
% Much research has focused on this problem by various methods, including pruning \cite{LLM-pruner}, quantization \cite{outlier_surpression_plus,Ominiquant}, and Knowledge Distillation (KD) \cite{Distllstepbystep, LLM-pruner,QUIP}.

As a compression method without modifying the model architecture, network quantization has shown promising effectiveness in reducing both memory requirements and inference latency. 
In the context of LLMs, quantization can be categorized into weight-only and weight-activation quantization.
While weight-activation quantization can leverage faster integer computations instead of floating-point computations, prior studies have identified significant outliers in activations, posing challenges for successfully quantizing activations to lower bit precision (e.g., 4 bits).
Consequently, weight-only quantization becomes a better trade-off between efficiency and accuracy, as the main bottleneck of deploying LLMs is memory bandwidth, which usually preserves more accuracy.
Most recent works have focused on weight-only quantization, successfully quantizing weights to 4 and 3 bits (even 2 bits) \cite{Ominiquant, QUIP}.
However,  weight-only quantization has to dequantize the low-bit weights to floating-point in real-time before performing multiplication with floating-point activation. 
% Therefore, while weight-only quantization primarily benefits from reduced memory access, it suffers from slower computation processes.

% In addressing the above limitations of weight-only quantization, ternarization provides advantages of both reduced memory consumption and energy-saving floating-point additions \cite{Energy}. Previous investigations into ternarized neural networks, as demonstrated in works such as \cite{TernaryBERT, RTN, STTN}, have primarily focused on convolutional neural networks and smaller encoder-based transformers like BERT. 
% % In this study, our attention shifts to applying ternarization to Large Language Models (LLMs). 
% However, employing previous ternarization methods directly in LLMs encounters new challenges due to outliers present in both weights and activations.

Fortunately, by replacing floating-point multiplications with energy-saving floating-point additions, ternarization can solve the above problems. However, existing ternarized works are dedicated to convolutional neural networks and smaller encoder-only transformers like BERT, which can not be directly applied to LLMs due to the following challenges:
Firstly, we observe asymmetric outliers and non-zero means in weights, indicating that previous symmetric ternarization methods are suboptimal. Secondly, extreme low-bit quantization leads to severe information loss in pretrained LLMs, including a narrowed feature representation range, loss of prominence in dominant channels, and disruption of the semantic clustering of related words. When trying to recover the lost information from the floating-point model, it is difficult to force the ternarized model to emulate the exact feature representation of the floating-point teacher due to the former's limited expressive capacity.

% It is crucial to highlight the essential role of outliers in knowledge distillation, underscoring the importance of considering them for achieving optimal solutions.

% In the context of ternarization methods in LLMs, it is important to address the challenges encountered in previous approaches. These challenges include the presence of asymmetric outliers and non-zero means in LLMs. Additionally, extreme low-bit quantization negatively impacts feature preservation by increasing variance while limiting outlier values. The mismatch between the distributions of the ternarized model and the floating-point model is another significant obstacle, especially due to the magnitude difference of outliers. This discrepancy renders the most common used MSE loss in knowledge distillation useless in achieving alignment. Furthermore, it is worth noting the essential role of outliers in knowledge distillation, emphasizing the need to consider them for optimal solutions. 

Based on the above discoveries, we propose two simple but effective methods: Dual Learnable Ternarization (DLT) and Outlier-Friendly Feature Knowledge Distillation (OFF). DLT is a custom ternary quantizer for weights with weird distribution in LLMs, enabling both scale and shift learnability. To recover semantic information from the original models, we introduce Outlier-Friendly Feature Knowledge Distillation (OFF). OFF aims to maximize the mutual information between the floating-point and quantized models while leveraging the insensitivity of cosine similarity to outliers, thereby diminishing training instability.
%which leverages the insensitivity of cosine similarity to outliers to mitigate training instability. 
TernaryLLM surpasses both previous post-training and quantization-aware methods on standard NLP benchmarks, including text generation and zero-shot tasks, using models from the OPT and LLaMA families.
Specifically, for one of the most powerful open-source models, LLaMA-3, our approach (W1.58A16) outperforms the previous quantization-aware training method (W2A16) by 5.8 in terms of average perplexity and by 8.2 in terms of average accuracy on zero-shot tasks.

Our contributions can be summarized as follows:
\begin{itemize}
    \item We observe in group-wise quantization, weights in the same group are asymmetric distributed. This phenomenon motivates us to propose Dual Learnable Ternarization (DLT) which enables learnable scales and shifts.
    
    % \item We found that cosine similarity, due to its incorporation of semantic information and insensitivity to outliers, is more suitable as the feature distance metric for LLMs. We further propose Outlier-Friendly Feature Knowledge Distillation (OFF).
    \item To recover semantic information from the original models, we introduce Outlier-Friendly Feature Knowledge Distillation (OFF), which utilizes the insensitivity of cosine similarity to outliers, helping to prevent training instability.
    
    \item We conducted experiments on text generation and zero-shot tasks using models from the OPT and LLaMA families. Our method outperforms both previous post-training and quantization-aware methods, even with fewer bits and a larger group size. 
\end{itemize}

\section{Related Work}
\subsection{LLM Quantization}
Quantization has found extensive application in accelerating models during inference \cite{Quantization_Jacob,Adaround,Brecq}. In the current era of burgeoning LLMs development, quantization has become widely employed for LLMs as well \cite{outlier_suppresion,Smoothquant,AWQ,GPTQ,LLM-QAT}. Based on whether activations are quantized, it can be classified into weight-only quantization and weight-activation quantization.

\textbf{Weight-Activation Quantization.}
Weight-activation quantization quantizes both floating-point weights and activations into low-bit integers. Much research has focused on the extremely large outliers in activations and designed approximate methods to alleviate this problem \cite{Smoothquant,outlier_suppresion,outlier_surpression_plus}. For instance, SmoothQuant \cite{Smoothquant} migrates the distribution imbalance from activations to weights and enables 8-bit weight, 8-bit activation (W8A8) quantization for LLMs. However, it is still challenging to quantize both weights and activations into low-bit, e.g., W4A4.

\textbf{Weight-Only Quantization.} Weight-only quantization only quantizes weights into low-bit while leaving activations floating-point. In the context of LLMs, the primary memory consumption results from model weights \cite{Ominiquant, GPTQ}. Weight-only quantization gains significant speedup and enables inference on consumer-level GPUs. However, in weight-only quantization, there is a need to dequantize the quantized weights to floating-point in real-time before computing with floating-point activation. Therefore, while weight-only quantization primarily benefits from reduced memory access, it suffers from slower computation processes. Fortunately, this problem can be naturally solved by pushing the quantization bit to ternary or binary, providing advantages of both reduced memory consumption and energy-saving floating-point additions \cite{Energy}. Previous work PB-LLM \cite{PB-LLM} only succeeds in partially quantizing weights to 1-bits and leaving salient weights at 8-bits. To distinguish between different data types, it incurs non-negligible overheads. In this paper, we aim to realize fully ternarized LLMs, relying on custom ternary quantization and Knowledge Distillation.

% \textbf{Quantization-Aware Training.} Weight-only quantization only quantizes weights into low-bit while leaving activations floating-point. In the context of LLMs, the primary memory consumption results from model weights \cite{Ominiquant, GPTQ}. Weight-only quantization gains significant speedup and enables inference on consumer-level GPUs. However, in weight-only quantization, there is a need to dequantize the quantized weights to floating point in real-time before computing with floating-point activation. Therefore, while weight-only quantization primarily benefits from reduced memory access, it suffers from slower computation processes. Fortunately, this problem can be naturally solved by pushing the quantization bit to ternary or binary, providing advantages of both reduced memory consumption and energy-saving floating-point additions \cite{Energy}. Previous work PB-LLM \cite{PB-LLM} only succeeds in partially quantizing weights to 1-bits and leaving salient weights at 8-bits. To distinguish between different data types, it incurs non-negligible overheads. In this paper, we aim to realize fully ternarized LLMs, relying on custom ternary quantization and Knowledge Distillation.

\subsection{Knowledge Distillation}

Knowledge distillation (KD) was initially proposed in \cite{KD_Hinton} to transfer knowledge from the logits of teacher models to student models. Later, feature distillation has been proposed to leverage information from hidden layers \cite{FitNets} instead of the output layer. In model quantization, particularly in low-bit settings, KD has been widely used to improve performance and narrow the performance gap between floating-point models and quantized models.

However, directly applying previous knowledge distillation methods in LLMs faces new challenges. The presence of outliers in features makes the previous KD method based on Mean Squared Error ineffective. It is important to recognize the crucial role of outliers in knowledge distillation, highlighting the necessity of fully utilizing them for optimal solutions while also avoiding their negative impact on training stability.

% The previous feature knowledge distillation metric, mean squared error (MSE), is too sensitive to significantly large outliers, making the training process unstable. Furthermore, we observe that these outliers play a crucial role in knowledge distillation; therefore, simply ignoring them leads to a suboptimal solution. 

\section{Background}
\textbf{Notification}. Transformers, including LLMs, are composed of a sequence of encoder or decoder layers. The notation $L$ is used to represent the layer number. For the $l$-th  decoder (or encoder) layer, we denote the input as $H_l$ and the output as $H_{l+1}$.

A linear layer in Transformer is defined as $Y = W \cdot X$, where $W \in \mathbb{R}^{C_o \times C_i}$ and $X \in \mathbb{R}^{C_i \times T}$. Here, $C_o$ and $C_i$ denote the output channel  and input channel, respectively, and $T$ represents the sequence length. To simplify, the number of elements in $W$ is denoted as $N = C_o \times C_i$.

\begin{figure}[t] %H为当前位置，!htb为忽略美学标准，htbp为浮动图形
    \centering %图片居中
    \includegraphics[width=1.0\textwidth]{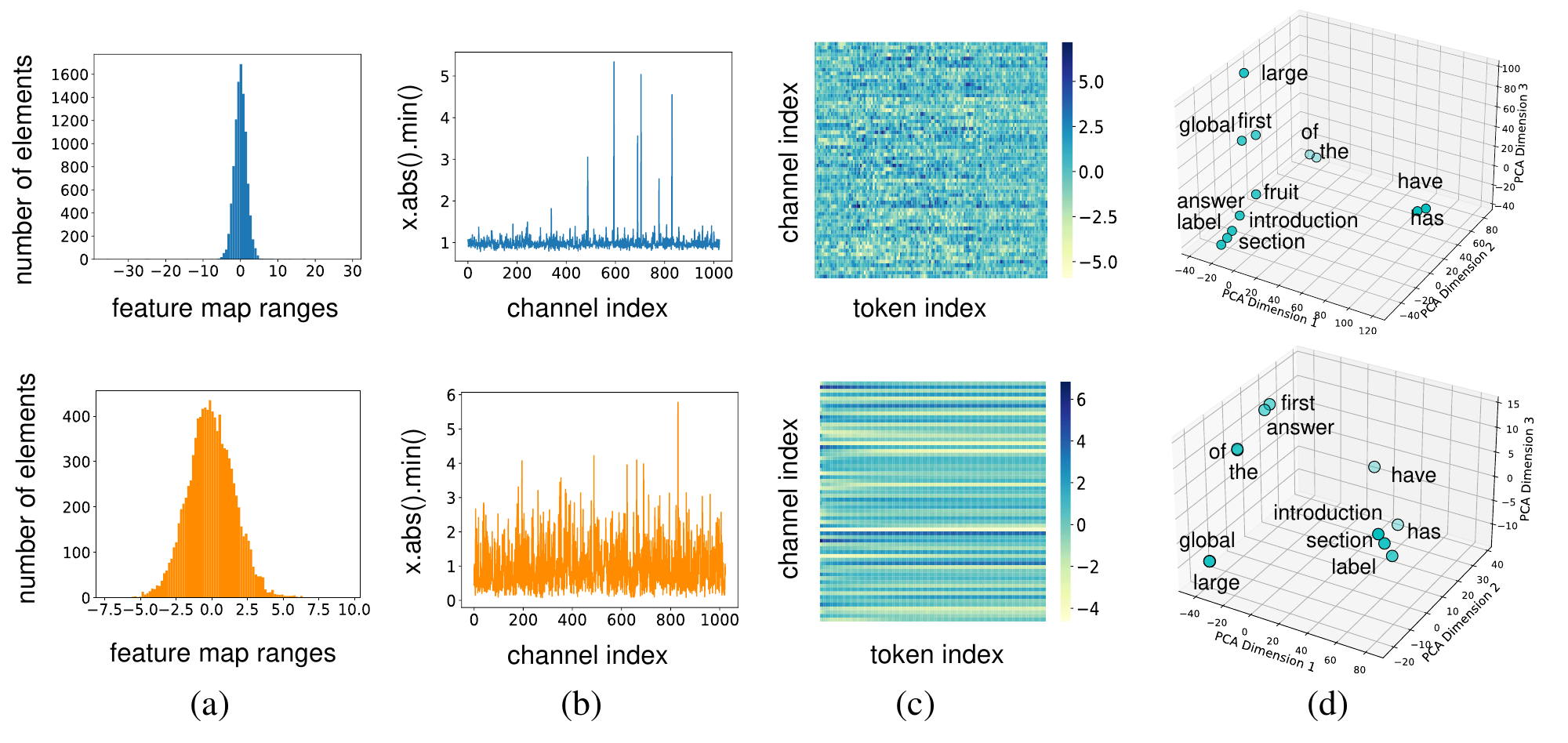} %插入图片，[]中设置图片大小，{}中是图片文件名
    \caption{
 An example of the features in the 23rd decoder layer to illustrate the problems incurred by extreme low-bit quantization. 
 The first and second lines correspond to the float-point and quantized models, respectively. Extreme low-bit quantization leads to severe information loss in pretrained LLMs, including a narrowed feature representation range (Figure~\ref{challenge2_fig} (a)), loss of prominence in dominant channels (Figure~\ref{challenge2_fig} (b)), and disruption of the semantic clustering of related words (Figure~\ref{challenge2_fig} (c) and (d)).
 % (a) Feature magnitude of float-point and quantized models. (b) Dominant channels of features. (c) Visualization of the activation patterns. (d) Word embedding in decode layer. More visualizations are provided in the Appendix.
 } 
    \label{challenge2_fig} %用于文内引用的标签
\end{figure}

\textbf{Weight Ternarization.} Ternarization converts floating-point weights in a model into three values. To project the $i$-th element of $W$, denoted as $W_i$, to $T_i \in \{-1, 0, 1\}$, a parameter threshold $\Delta$ is used:
\begin{equation}
\label{eq_ternary_w}
    \begin{aligned}
        T_i=\begin{cases}
  -1, &\quad \text{ if } W_i< -\Delta \\
  0, &\quad\text{ if } |W_i|\leq  \Delta \\
  +1, &\quad\text{ if } W_i> \Delta
\end{cases}
    \end{aligned}    
\end{equation}
To enhance performance, TWN \cite{TWN} introduces a scaling factor $\alpha$ to estimate the magnitude of the original weights:
\begin{equation}
    \label{mse_target}
    \begin{aligned}
        \min_{\alpha}\sum_i^{N}||W_i - \alpha \cdot  T_i||^2
    \end{aligned}    
\end{equation}

The exact solution for $\Delta$ and $\alpha$ is time-consuming. As an alternative method, they approximate $\Delta$ as:
\begin{equation}
    \begin{aligned}
        \Delta=\frac{0.7\cdot \sum_i^{N} |W_i|}{N}
    \end{aligned}
\end{equation}
After determining $W^t$, the optimal scale factor $\alpha$, which minimizes \eqref{mse_target}, can be obtained as follows:
\begin{equation}
\label{eq_twn_alpha}
    \begin{aligned}
        \alpha^*=\frac{ \sum_i^{N} T_i \cdot W_i}{ 
        \sum_i^{N}{ |T_i|} }
    \end{aligned}
\end{equation}

% \begin{figure}[t] %H为当前位置，!htb为忽略美学标准，htbp为浮动图形
%     \centering %图片居中
%     \includegraphics[width=1.0\textwidth]{figures/architecture.pdf} %插入图片，[]中设置图片大小，{}中是图片文件名
%     \caption{This is a $\tau$} %最终文档中希望显示的图片标题
%     \label{Fig.main1} %用于文内引用的标签
% \end{figure}%结束环境

% \begin{figure*}[t]
% \vskip 0.2in
% \begin{center}
% % \centerline{\includegraphics[width=\columnwidth]{figures/asym_weight.pdf}}
% \centering
%  \vspace{-20pt} % 设置垂直间距，可以根据需要调整数值
%  \subfigure[Boxplot]
%  {\includegraphics[width=0.45\textwidth]{figures/boxplot.pdf}}
%  \subfigure[Histogram]{\includegraphics[width=0.45\textwidth]{figures/channel_35_delta.pdf}}
 
% \caption{Example of weight Distribution of OPT-125M. (a) The Weights contain asymmetric outliers (represented by circles) and exhibit non-zero means (horizontal lines within the boxes).
% (b) Representing values within the range $[-\Delta, \Delta]$ as zero, instead of utilizing a learnable parameter, is considered suboptimal. }
% \label{fig_weight_asym}
% \end{center}
% \vskip -0.2in
% \end{figure*}

\section{Challenges of Ternarizing LLMs}
\label{sec_difficulty}

\begin{wrapfigure}{r}{0.6\textwidth} % 'l' 表示图片在左边，'0.4\textwidth' 表示图片占页面宽度的 40%
    \centering
    \includegraphics[width=\linewidth]{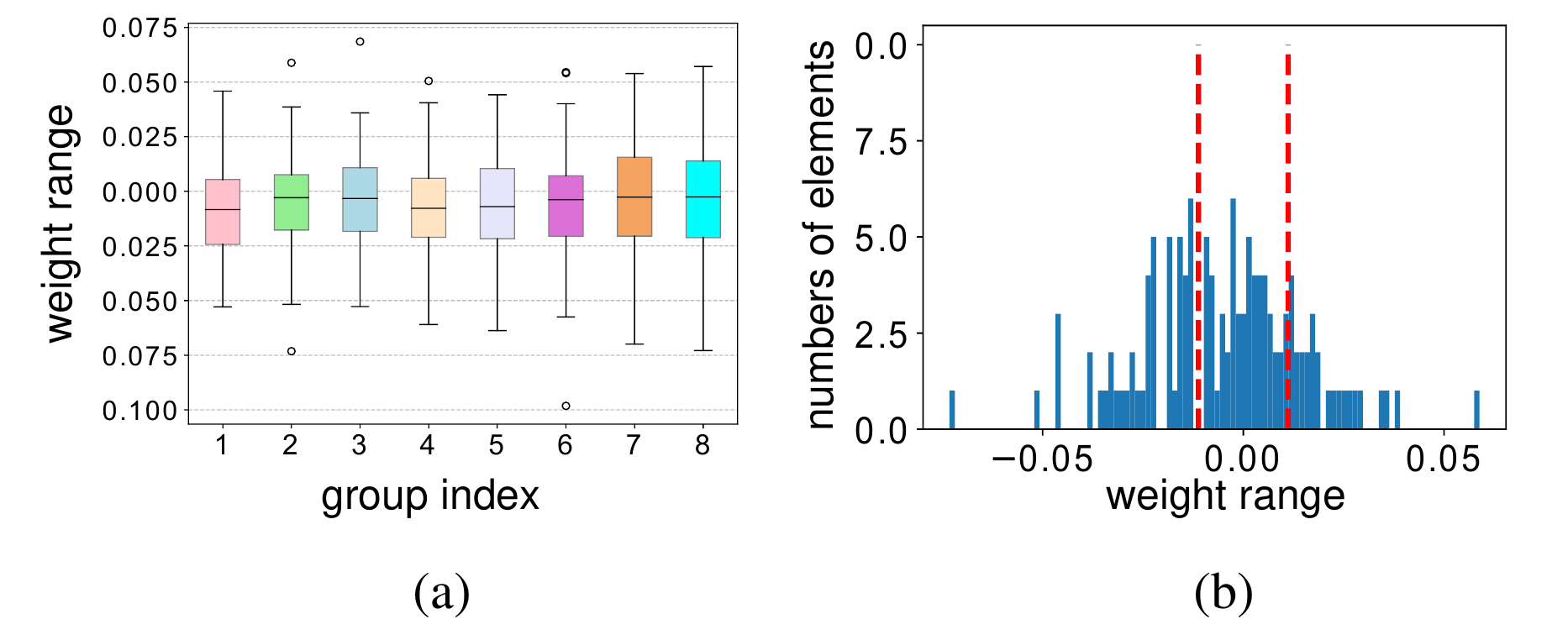} % 替换为你的图片文件名
    \caption{The weights in certain groups display noticeable asymmetric outliers and a non-zero mean distribution. }
    \label{fig_weight_asym}
 \end{wrapfigure}

In this section, we conduct in-depth analyses of the challenges encountered in implementing ternarized LLMs. Firstly, we have observed that the previous weight ternarization method is inadequate for handling asymmetric outliers and non-zero mean distributions. Secondly, extremely low-bit quantization results in substantial information loss, necessitating feature knowledge distillation. Based on preliminary experiments, we have found it difficult to directly align features between quantized and original models due to their different expressive capacities.

\textbf{Challenge 1. For group-wise quantization, weights in certain groups display noticeable asymmetric distribution with non-zero mean.}
For transformers, especially LLMs, Outlier Surpression+ \cite{outlier_surpression_plus} has identified asymmetric behavior among channels in \textbf{features}. 
This paper finds a similar phenomenon in \textbf{weights}.
We illustrate this observation with the query weights in the OPT model. In Figure~\ref{fig_weight_asym}(a), weights in some groups exhibit a negative mean but significantly positive outliers, while the others show the opposite behavior. Previous methods, such as TWN, lead to a suboptimal trade-off between clamping error and round error on both sides: the scale factor in Equation~\eqref{eq_twn_alpha} may be too large for negative values but too small for positive ones. Moreover, in Figure~\ref{fig_weight_asym}(b), due to this asymmetry, the mean of values located in intervals $[-\Delta, \Delta]$ is negative,  which leads to those methods to quantize these values to zero being suboptimal.

\textbf{Challenge 2. Extreme low-bit quantization leads to severe loss of information in pretrained LLMs.} The distribution depicted in Figure~\ref{challenge2_fig} highlights the challenges encountered in ternary quantization. Figure~\ref{challenge2_fig} (a) illustrates that the feature representation range is considerably narrowed after ternarization, indicating a constrained expressive power of the quantized model. Channels that were originally dominant, meaning those with larger magnitudes, become overwhelmed by other channels and are no longer prominent, as shown in Figure~\ref{challenge2_fig} (b).
Furthermore, in Figure~\ref{challenge2_fig} (c) and (d), we observe that after pretraining on large datasets, LLMs tend to cluster semantically related words tightly (for example, the words "has" and "have") \cite{LLMwowrdembedding}. However, ternarization destroys this semantic structure, causing different words to be distributed in an out-of-order manner. Therefore, directly employing quantization-aware training on LLMs wastes the pretrained information and may lead to a different convergence point, risking overfitting.

\textbf{Challenge 3. It is difficult to force the ternarized model to emulate the exact feature representation of the floating-point due to the former's limited expressive capacity.} Focused on Challenge 2, we propose the utilization of feature knowledge distillation (Feature KD). Feature KD can introduce more supervised signals from the pretrained model to realign the ternarized model with the floating-point model. However, as mentioned in Figure~\ref{challenge2_fig}, the values of some channels are significantly larger than others. If they are not correctly balanced, the small gradients produced by the less-activated channels can be overwhelmed by the large gradients from the dominant ones. In Figure~\ref{challenge3_fig}, we observe that the previously popular distillation metric, mean square error (MSE), suffers from severe oscillations. This phenomenon arises due to the inherent heterogeneity between the floating-point and ternarized models' features. Unlike the real-numbered weights of the floating-point model, the ternarized model's weights are constrained to a discrete set of three values. Consequently, the latter fails to replicate the exact feature representation achieved by the former.

% To study this problem, we identified a similar phenomenon in feature knowledge distillation . In Figure 3(a), we firstly found previous feature KD metric, mean squared error is too sensitive to  siginifcantly large outliers, making the training process unstable (Figure~\ref{fig_training_loss} (a)). To thoroughly investigate this phenomenon, we try to clamp these large outliers or skipping iteration  an abnormal loss. Specifically, we record the historical loss using exponential moving average:
% \begin{equation}
%     \begin{aligned}
%         \mu=\varepsilon \cdot MSE(H_l^T,H_l^S) +(1-\varepsilon ) \cdot \mu
%     \end{aligned}
% \end{equation}

% Subsequently, if the current loss significantly surpasses the historical loss $\mu$ , we deactivate feature KD in the ongoing iteration. The experimental results are presented in Figure~\ref{fig_training_loss}(b-c). While these approaches can alleviate the impact of outliers to some extent, they fall short of fully exploiting the information within the features. Additionally, they exhibit sensitivity to the choice of hyperparameters.

% \begin{equation}
%     \begin{aligned}
%         \mathcal{L}_{feat}^l=\begin{cases}
%   & MSE(H_l^T,H_l^S), \quad \text{ if } MSE(H_l^T,H_l^S)< \u^l\\
%   & 0, \quad\text{ if } |W_i|\leq  \Delta \\
%         \end{cases}
%     \end{aligned}
% \end{equation}

\begin{figure}[t]
    \centering
    \includegraphics[width=1\linewidth]{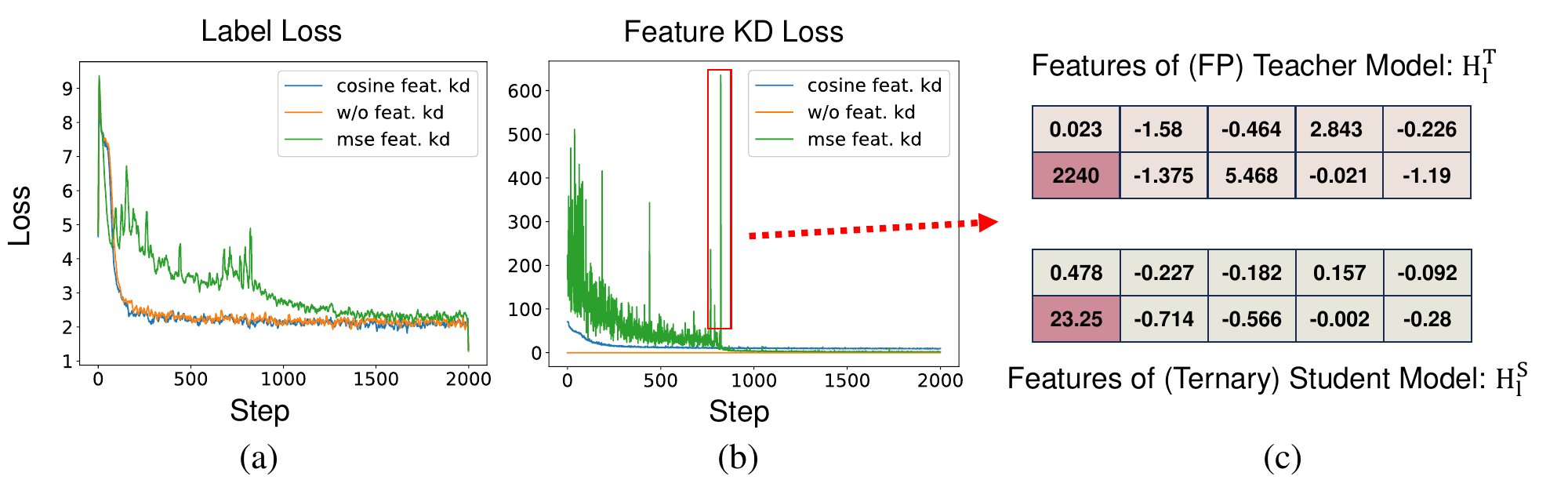}
    \caption{Feature knowledge distillation results for LLaMA-1-7B. Cosine similarity is less sensitive to outliers in features compared to MSE. (a) Ground truth loss of the training. (b) Feature knowledge distillation loss of the training. (c) The reasons for severe oscillations in MSE distillation.
    }
    \label{challenge3_fig}
\end{figure}

% \begin{table*}[bt]
% \caption{C4 perplexity of different methods in OPT and LLaMa models}
% \label{Table_C4}
% \vskip 0.15in
% \begin{center}
% \begin{small}
% % \begin{sc}
%   \begin{tabular}{lllccccl}
%     \hline
%     PPL & WBits &Method& OPT-125M&OPT-1.3B&OPT-2.7B &OPT-6.7B &LLaMA-7B \\
%     \hline
%      FP & 16 & - & 24.60 &14.72 &13.16 &11.74& 7.08 \\
%      \hline
%     W2g64 & 2.5& RTN& 3.9e3 &7.3e3 &1.2e5 &6.3e3&151.43  \\
%      & & GPTQ &  133.51& 31.31& 23.23& 16.24 & 17.71\\
%      &  & AWQ  & 90.19 &27.34& 20.01& 15.20 &  2.8e5\\
%      &  & OmniQuant &  64.01 &23.71 &19.16& 15.44 &  11.78 \\
%      & & PB-LLM &  - & -&-&-&47.09  \\
%      & & DB-LLM &  - & -&-&-& 9.74 \\
%     \hline
%     W2g128 & 2.25 & RTN&  5.0e3 &7.7e3& 3.8e4& 5.2e3 &   1.0e3 \\
%      & & GPTQ & 597.66& 60.88 &33.83& 18.55 &27.71 \\
%      &  & AWQ  &168.35 &38.38& 26.41 &16.48 & 1.9e5 \\
%      &  & OmniQuant &  80.10 &27.33& 21.11& 16.67 & 12.97 \\
%     \hline
%     W2 & 2& OmniQuant & 4.39e4 &68.10  & 43.47&  35.19&   24.89 \\
%     W2 & 2&Ours &32.57  &18.01 &17.49&14.51&9.84\\
%     \hline
%   \end{tabular}
% % \end{sc}
% \end{small}
% \end{center}
% \vskip -0.1in
% \end{table*}

\section{Method}

In Section~\ref{sec_difficulty}, we explore the difficulties in ternarizing LLMs. In this section, based on the above findings, we propose two simple but effective methods to boost performance.

\subsection{Dual Learnable Ternarization}

In Section~\ref{sec_difficulty}, we observed that the weights exhibit significant asymmetric outliers and a non-zero mean distribution. Focusing on this phenomenon, we propose \textbf{Dual Learnable Ternarization (DLT)}, which not only enables learnable quantization scales but also shifts.

Specifically, we first map float weights $W_i$ into ternarized values $T_i$ using threshold $\Delta$ as the same as TWN:
    \begin{equation}
\label{eq_ternary_w_2}
    \begin{aligned}
        T_i=\begin{cases}
  -1, &\quad \text{ if } W_i< -\Delta \\
  0, &\quad\text{ if } |W_i|\leq  \Delta \\
  +1, &\quad\text{ if } W_i> \Delta
\end{cases}
    \end{aligned}    
\end{equation}

Instead of using the mean of weights as $\alpha$ in equation~\ref{eq_twn_alpha}, we set $\alpha$ as a learnable parameter that can be updated during training. To further adapt to the asymmetric distribution, we introduce a learnable shift parameter $\gamma$:
\begin{equation}
\begin{aligned}
    D_i \approx \alpha \cdot T_i +\gamma
\end{aligned}    
\end{equation}

The gradients of parameters of $\alpha$ and $\gamma$ can been computed as follows:
\begin{equation}
\begin{aligned}
    \frac{\partial \mathcal{L}}{\partial \alpha}=
 \sum_{ i:|W_i|<\Delta}   \frac{\partial \mathcal{L}}{\partial D_i } 
\end{aligned}
\end{equation}

\begin{equation}
\begin{aligned}
    \frac{\partial \mathcal{L}}{\partial \gamma}=
 \sum_{ i}   \frac{\partial \mathcal{L}}{\partial D_i} 
\end{aligned}
\end{equation}

By utilizing the Straight-through estimator (STE) \cite{STE}, we can approximate the gradients to float weight $W_i$:
\begin{equation}
\begin{aligned}
    \frac{\partial \mathcal{L}}{\partial W_i}=
    \begin{cases}
           \alpha \cdot  \frac{\partial \mathcal{L}}{\partial D_i}, &\quad \text{ if } W_i > \Delta \\
    1\cdot \frac{\partial \mathcal{L}}{\partial D_i}, &\quad\text{ if } |W_i|\leq  \Delta \\
   -\alpha \cdot  \frac{\partial \mathcal{L}}{\partial D_i}, &\quad\text{ if } W_i> \Delta
    \end{cases}
\end{aligned}
\end{equation}
% \textcolor{red}{Here needs a summarization of our advatanges?}

\textbf{Efficiency Analysis.} By employing DLT, we can replace original floating-point multiplications with just two additions. 
\begin{equation}
    \begin{aligned}
        y=\sum_i D_i \cdot x_i  = \alpha \cdot   \sum_i (T_i \cdot x_i) +  \gamma \cdot \sum_i x_i  
    \end{aligned}
\end{equation}
Because the second summation is independent of weights, the total number of additions in our method is $\mathcal{O}(LC_o C_i + LC_i)$, and the total number of multiplications is $\mathcal{O}(2G)$, where $G$ is the total number of groups.

\subsection{Outlier-Friendly Feature Knowledge Distillation}
Knowledge distillation, as demonstrated by Hinton \cite{KD_Hinton}, has proven to be effective in compressing CNNs and transformer models, particularly in scenarios with extremely low bit precision. In Section~\ref{sec_difficulty}, we identify the necessity of recovering semantic information in features of ternarized models. Here, from the perspective of mutual information, we propose Theorem~\ref{theorem1}, which explains how to maximize the mutual information between features in ternarized and floating-point models using cosine similarity. The proof process can be found in the appendix.

\begin{theorem}
\label{theorem1}
% Assume \(x \in \mathbb{R}^{C_i} \sim \mathcal{N}(0, \sigma_x^2)\), \(W = (w_1^T, w_2^T, \ldots, w_{C_o})^T\) and \(y = Wx\). Let \(W^q\) denotes the ternarization of \(W\) and \(y^q = W^q x\). The objective to maximize the mutual information between \(y\) and \(y^q\) (formally, \(\max I(y, y^q)\)) is achieved when cos(${y^q}, y)=1 $.

Assume \(x \in \mathbb{R}^{C_i} \sim \mathcal{N}(0, \sigma_x^2)\), \(W = (w_1^T, w_2^T, \ldots, w_{C_o})^T\) and \(y = \text{RMSNorm}(Wx)\). Let \(W_q\) denotes the ternarization of \(W\) and \(y_q = \text{RMSNorm} (W_q x)\). The objective to maximize the mutual information between \(y\) and \(y_q\) $I(y, y_q)$ can been achieved by $\mathbb{E}_{p(x)}$cos(${W_qx}, Wx)=1 $.

\end{theorem}
Utilizing cosine similarity to maximize the mutual information between floating-point model and quantized model, we introduce a more suitable feature KD method termed \textbf{Outlier-Friendly Feature KD (OFF)} for LLMs. Specifically, we compute the cosine similarity between each token of teachers and students individually, then aggregate the total $T$ tokens to compute the feature KD loss for the current layer. Subsequently, we aggregate the first $L'$ layers to compute the total feature KD loss.
% Thus, utilizing cosine similarity as the distance metric, we propose a more suitable feature KD method called \textbf{Outlier-friendly Feature KD (OFF)} for LLMs. Specifically, we calculate the cosine similarity between the $i$-th token of teachers and students individually and subsequently summarize the total $T$ tokens as the feature KD loss of the current layer. We then summarize the first L' layers as the total feature KD loss.
\begin{equation}
    \begin{aligned}
\text{Cosine Similarity}(H^{Teach}\textit{,}H^{Stu}) = \sum_i^T \frac{ {h_i^{Stu} }^T \cdot h_i^{Teach} }{\|h_i^{Teach}\| \cdot \|h_i^{Stu}\|}
    \end{aligned}    
\end{equation}

% For the entire model, the feature KD loss is computed as the summation across all layers:
% \vspace{-1em} % 增加1个文本行高度的垂直间距
% \begin{equation}
%     \begin{aligned}
%         \mathcal{L}_{feat}=\sum_l^{L}1-\text{Consine Similairty}(H_l^{T}\textit{,}H_l^{S})
%     \end{aligned}
% \end{equation}

\begin{equation}
    \mathcal{L}_{feat}  = \sum_l^{L'} \text{Cosine Similarity}(H^{Teach}_l \textit{,} H^{Stu}_l)
\end{equation}

 The advantage of using cosine similarity as the distance metric is two folds: 
 % (a) it can force the student models to recover semantic (b) \textcolor{red}
 (a) By maximizing cosine similarity, the quantized model restores the semantic information of the original model. This ensures the quantized model captures the essential features learned by the floating-point model, leading to better generalization performance on unseen data. (b) Cosine similarity is scale-invariant, meaning it is not affected by the magnitude of the vectors but only by their direction. This property guarantees the distillation process's robustness, even in the presence of numerous outliers in the floating-point model.

% \begin{figure}[h]
% \vskip 0.2in
% % \captionsetup[subfigure]{labelformat=empty} 
% \begin{center}
% % \centerline{\includegraphics[width=\columnwidth]{figures/asym_weight.pdf}}\captionsetup[subfigure]{labelformat=empty}
% \centering
% \includegraphics[width=0.45\textwidth]{figures/cosine_vs_mse.pdf} % 替换为你的图片文件名
% \caption{An example to demonstrate the difference between cosine similarity and MSE. The cosine similarity is inherently less sensitive to absolute values and focuses more on relative values }
% \label{fig_cosine_vs_mse}
% \end{center}
% \vskip -0.2in
% \end{figure}

\textbf{Logits Knowledge Distillation.}
Besides the feature KD, we also distill
knowledge in the last layer by
 minimizing the soft cross-entropy (SCE) between
 quantized student logits $Z^{Stu}$ and teacher logits $Z^{Teach}$, i.e.,
\begin{equation}
    \begin{aligned}
        \mathcal{L}_{logits}=\text{SCE}(Z^{Stu},Z^{Teach})
    \end{aligned}
\end{equation}

\textbf{Objective Function.}
The overall objective function in the training process of TernaryLLM is as follows:
 \begin{equation}
     \begin{aligned}
         \mathcal{L}_{total} =\mathcal{L}_{label} +  \epsilon \cdot \mathcal{L}_{logits} + \delta \cdot \mathcal{L}_{feat} 
     \end{aligned}
 \end{equation}
Here, $\epsilon$ and $\delta$ represent the loss balancing scales, which can be found in the training details section. Additionally, $\mathcal{L}_{label}= SCE(Z^{Stu}, Z^{label})$ denotes the cross-entropy between $Z^{Stu}$ and ground truth $Z^{label}$.

\section{Experiments}
In this section, we evaluate our methods on both language generation and 
zero-shot  zero-shot tasks with OPT \cite{OPT} models and LLaMA
\cite{LLaMa} family models.

\begin{table}[t]
    \small
    \centering
    \caption{Evaluation results of weight-only quantization on the \textbf{LLaMA-3-8B} model. \#W indicates weight bits, \#G indicates group size, and ``-'' denotes per-channel quantization.}
    \label{llama-3-result}
    \setlength{\tabcolsep}{.95mm}
    % \resizebox{\linewidth}{!}
    {
    \begin{tabular}{lllccccccccc}
        \toprule
        \multirow{2}{*}{\textbf{Method}} & 
        \multirow{2}{*}{\textbf{\#W }} & 
        \multirow{2}{*}{\textbf{\#G}}&
        \multicolumn{3}{c}{\textbf{PPL$\downarrow$}} & 
        \multicolumn{6}{c}{\textbf{Zero-Shot$\uparrow$}} 
        \\
        \cmidrule(lr){4-6} \cmidrule(lr){7-12} 
        ~ & ~ &   ~ & \textbf{WikiText2} & \textbf{C4} & \textbf{PTB} & \textbf{PIQA} &\textbf{ARC-e} & \textbf{ARC-c} &\textbf{HellaSwag} & \textbf{Wino} & \textbf{Avg.} \\ 
        \midrule
        \textbf{LLaMA-3} & 16  & - & 6.1 & 9.2 & 10.6 & 79.9 & 80.1 & 50.4 & 60.2 & 72.8 & 68.6 \\
        \midrule
        
        \multirow[t]{2}{*}{RTN} 
         & 2  & 128 & 1.9E3 & 2.5E4 & 1.8E4  
         & 53.1 & 24.8 & 22.1 & 26.9 & 53.1 & 36.0 \\
         & 2  & - & 2.7E6 & 7.4E6 & 3.1E6  
         & 53.1 & 24.7 & 21.9 & 25.6 & 51.1 & 35.3 \\

        \multirow[t]{2}{*}{GPTQ} 
         &  2    &  128  &  2.1E2  &  4.1E4  &  9.1E2 %207.61 908.80
         & 53.9 & 28.8 & 19.9 & 27.7 & 50.5 & 36.2  \\
         &  2    &  -  &  5.7E4  &  1.0E5  &  2.7E5 
         & 52.8 & 25.0 & 20.5 & 26.6 & 49.6 & 34.9\\

        \multirow{1}{*}{AWQ} 

        & 2 & -  & 8.2E5 & 8.1E5 & 9.0E5 & 55.2 & 25.2 & 21.3 & 25.4 & 50.4 & 35.5 \\
    
        \multirow{1}{*}{QuIP} 
 
        & 2 &  -  & 85.1 & 1.3E2 & 1.8E2 & 52.9 & 29.0 & 21.3 & 29.2 & 51.7 & 36.8 \\ % 131.33 178.27
        \midrule

        DB-LLM & 2& 128 & 13.6& 19.2 & 23.8 & 68.9 & 59.1 & 28.2 & 42.1 & 60.4 & 51.8 \\

        \multirow[t]{2}{*}{PB-LLM}
        & 2 & 128 & 24.7 & 79.2 & 65.6  
        & 57.0 & 37.8 & 17.2 & 29.8 & 52.5 & 38.8\\
        &  1.7  &   128  &  41.8  &  2.6E2  &  1.2E2 % 261.87 117.48
        & 52.5 & 31.7 & 17.5 & 27.7 & 50.4 & 36.0 \\
        % \midrule

        %\rowcolor[gray]{0.9}
        %\multirow[t]{2}{*}{Our}
        % & \textbf{1.58} & 128 & 13.8 & \textcolor{red}{TODO} & \textbf{20.6}  
        % & \textbf{74.4} & \textbf{59.7}  & \textbf{34.3} & \textbf{62.7} & \textbf{62.5} & \textbf{58.7}  \\

         \rowcolor[gray]{0.9}
        Our & \textbf{1.58} & - & \textbf{11.2} & \textbf{13.4}& \textbf{16.3}& \textbf{73.7}& \textbf{61.2}& \textbf{36.4}& \textbf{63.9}& \textbf{65.0}& \textbf{60.0}\\

        \bottomrule
        
    \end{tabular}}
\end{table}

\begin{table}[t]
    \small
    \centering
    \caption{C4 and WikiText2 perplexity  $\downarrow$ of different methods in OPT models. \#W indicates weight bits, \#G indicates group size, and ``-'' denotes per-channel quantization.}
    \label{opt-result}
    \setlength{\tabcolsep}{.95mm}
    \resizebox{\linewidth}{!}
    {
    \begin{tabular}{lllcccccccc}
    \toprule
    
       \multirow{2}{*}{\textbf{Method}} & 
        \multirow{2}{*}{\textbf{\#W }} & 
        \multirow{2}{*}{\textbf{\#G}} & \multicolumn{2}{c}{\textbf{OPT-125M}} &  \multicolumn{2}{c}{\textbf{OPT-1.3B}} &  \multicolumn{2}{c}{\textbf{OPT-2.7B}}  &  \multicolumn{2}{c}{\textbf{OPT-6.7B}}\\
    \cmidrule(l){4-5} \cmidrule(l){6-7} 
\cmidrule(l){8-9} \cmidrule(l){10-11}
 & & & \textbf{WikiText2} & \textbf{C4} & \textbf{WikiText2} & \textbf{C4} & \textbf{WikiText2} & \textbf{C4} & \textbf{WikiText2} & \textbf{C4}  \\

    \midrule
     FP & 16 & - & 27.65& 24.60 &14.63 &14.72&12.47 &13.16&10.86 &11.74 \\
     \midrule 
     
   \multirow[t]{2}{*}{RTN}  & 2& 64 & 7.0e3& 3.9e3 &1.0e4& 7.3e3 &19.3e4&  1.2e5 &7.6e3 &6.3e3\\
    & 2& 64 & 7.0e3 & 5.0e3 & 1.0e4 &7.7e3& 19.3e4 &3.8e4& 7.6e3 &5.2e3\\

    \multirow[t]{2}{*}{GPTQ} & 2 & 64& 204.40 & 133.51& 49.58 &31.31& 29.37 &23.23& 16.81 &16.24 \\
    & 2&128& 597.66  & 597.66& 115.16 &60.88 & 61.59&33.83& 20.18 &18.55\\

  \multirow[t]{2}{*}{AWQ} & 2 & 64& 124.18 & 90.19 & 29.78 &27.34&  20.64&20.01& 14.63 &15.20 \\
    & 2&128& 251.84 &168.35& 47.97 &38.38& 28.50 &26.41 & 16.20 &16.48\\
% &2 &64 & 62.56 &64.01  & 21.40 &23.71& 16.76 &19.16& 13.57  &15.44\\
    \multirow[t]{2}{*}{OmniQuant}& 2 & 128&  75.43 & 80.10 &  23.95 &27.33& 18.13 &21.11& 14.43 &16.67 \\
    &2&-& 5.18e4 & 4.39e4 & 48.22  &68.10& 31.04  & 43.47& 22.77  &35.19\\

     \rowcolor[gray]{0.9}
    Our & \textbf{1.58} & -   & \textbf{39.92} &\textbf{32.57} & \textbf{18.51}  &\textbf{18.01} & \textbf{17.98} &\textbf{17.49}& \textbf{13.81} &\textbf{14.51} \\
    \bottomrule
  \end{tabular}}
\end{table}

\begin{table}[t]
    \small
    \centering
    \caption{Evaluation results of weight-only quantization on the \textbf{LLaMA-1-7B} and \textbf{LLaMA-2-7B} model. \#W indicates weight bits, \#G indicates group size, and ``-'' denotes per-channel quantization.}
    \label{llama-2-result}
    \setlength{\tabcolsep}{.95mm}
    \resizebox{\linewidth}{!}
    {
    \begin{tabular}{lllccccccccc}
        \toprule
        \multirow{2}{*}{\textbf{Method}} & 
        \multirow{2}{*}{\textbf{\#W }} & 
        \multirow{2}{*}{\textbf{\#G}}&
        \multicolumn{3}{c}{\textbf{PPL$\downarrow$}} & 
        \multicolumn{6}{c}{\textbf{Zero-Shot$\uparrow$}} 
        \\
        \cmidrule(lr){4-6} \cmidrule(lr){7-12} 
        ~ & ~ &   ~ & \textbf{WikiText2} & \textbf{C4} & \textbf{Avg.} & \textbf{PIQA} &\textbf{ARC-e} & \textbf{ARC-c} &\textbf{HellaSwag} & \textbf{Wino} & \textbf{Avg.} \\ 
        \midrule
        \textbf{LLaMA-2} & 16 & - & 5.47 & 6.97 &6.22 & 76.99 & 53.58 & 40.53 & 72.96 & 67.25 & 62.26 \\
        \midrule        
    {GPTQ} & 2 & 64 & 8.97 & 13.25 & 11.11 & 68.39 & 42.13 & 31.91 & 54.64 & 58.64 & 51.14 \\
     AWQ & 2 & 64 & 2.06e5 & 1.54e5 & 1.8e5   & 50.00 & 26.52 & 26.79	& 26.14 & 49.64 & 35.82 \\
    OmniQuant& 2 & 64 & 9.64 & 12.73 & 11.19 & 68.72 & 39.77 & 30.89 & 53.44 & 56.12 & 49.79 \\

    \midrule
     PB-LLM & 2 & 64 & 20.37 & 44.88 & 32.63 & 55.22 & 29.88 & 22.01 & 30.49 & 50.36 & 37.59  \\
     DB-LLM & 2 & 64 & \textbf{7.23} & {9.62} & {8.43} & \textbf{73.18} &{ 45.20} & {33.53}	& {61.98} & \textbf{61.72} & {55.12} \\

         \rowcolor[gray]{0.9}
       \multirow[t]{2}{*}{Our}  & \textbf{1.58} & - &  7.46 & \textbf{9.16} & \textbf{8.31} & 72.47 &{46.46} & 33.44	& {63.84} & 60.93 & {55.42} \\

        \rowcolor[gray]{0.9}
       & \textbf{1.58} & 64 &  7.7& 9.45& 8.57  & 72.68& \textbf{48.06} &\textbf{34.89}& \textbf{63.94}& \textbf{61.72} & \textbf{56.25} \\
 % & \textbf{1.58} & - &  7.46 & \textbf{9.16} & \textbf{8.31} & 72.47 &{46.46} & 33.44	& {63.84} & 60.93 & {55.42} \\
\midrule \midrule

\textbf{LLaMA-1} & 16 & - & 5.68  &  7.08 & 6.38 & 77.37 & 52.53 & 41.38 & 72.99 & 66.85 & 62.22 \\
\midrule
GPTQ & 2 & 64 & 22.10  & 17.71  & 19.91&   59.36 & 32.11 & 25.09 & 35.14 & 49.01 & 40.14 \\
AWQ & 2 & 64 &   2.5e5  & 2.8e5 & 5.3e5&  50.05 & 25.76 & 29.44	& 25.93 & 49.96 & 36.23 \\
OmniQuant & 2 & 64 &  8.91 & 11.79&   10.4 &68.66 & 44.49 & 29.69 & 54.32 & 55.56 & 50.54 \\

\midrule
PB-LLM & 2 & 64 & 20.61 &47.09&33.85  &  55.39 & 34.22 & 24.23 & 31.99 & 52.88 & 39.74 \\
DB-LLM & 2 & 64 &7.59 &9.74&8.67&  {72.14} & {44.70} & {33.62}	& {60.71} & {61.01} & {54.44} \\

 \rowcolor[gray]{0.9}
\multirow[t]{2}{*}{Our} & \textbf{1.58} & - &7.82& 9.52& 8.67& 72.74& \textbf{46.89} & 34.98& 62.40& 59.35& 55.27 \\

 \rowcolor[gray]{0.9}
& \textbf{1.58} & 64 & \textbf{7.48} & \textbf{9.38} & \textbf{8.58} & \textbf{73.06} & 45.49& \textbf{35.15} &\textbf{63.78}&\textbf{62.58} & \textbf{56.01} \\

\bottomrule
 
\bottomrule
    \end{tabular}}
\end{table}

\subsection{Experiment Setup}
\textbf{Dataset.}
Following the previous work \cite{PB-LLM} we use RedPajama \cite{redpajama} as the training dataset. RedPajama is an open-source reproduction of the pre-training data for LLaMA\cite{LLaMa}. This mainly consists of web text (Common Crawl and C4 \cite{C4}) and high-quality sources such as arXiv and Stack Exchange. 

\textbf{Training Details.}
We utilize a full-precision model to initialize ternarized models. We use the AdamW optimizer with zero weight decay to optimize the learnable parameters. The batch size is 16. We use cosine learning rate decay, and the total number of iterations is 10,000 steps. More training details can be found in the Appendix. 

\textbf{Evaluation Tasks.}
 We evaluate our methods on both language generation and zero-shot tasks. We report the perplexity on WikiText2 \cite{wikitext-2}, PTB \cite{PTB} and C4 \cite{C4}. For zero-shot tasks, we provide accuracy on  datasets including PIQA \cite{PIQA}, ARC \cite{ARC}, BoolQ \cite{BoolQ}, HellaSwag \cite{HellaSwag}  and Winogrande \cite{WinoGrande}.

\subsection{Results on Language Generation}
Experiments were conducted on OPT models \cite{OPT} ranging from 125M to 6.7B, as well as on 7B versions of the LLaMA-1, LLaMA-2, and LLaMA-3 models \cite{LLaMa}. In this comparison, we evaluate our method against previous weight-only techniques, including post-training quantization methods such as OmniQuant \cite{Ominiquant}, AWQ \cite{AWQ}, and GPTQ \cite{GPTQ}, as well as quantization-aware training methods such as PB-LLM \cite{PB-LLM} and DB-LLM \cite{DBLLM}. PB-LLM retains significant weights in higher bits, whereas DB-LLM employs two binary matrices to represent the weights. The average bit-width for both methods is 2 bits.
The perplexity results for LLaMA family models are presented in Table~\ref{llama-3-result} and Table~\ref{llama-2-result}. 
Notably, even with an average bit width of only 1.58 , our methods surpass the previous 2-bit quantization-aware methods in average perplexity, for example, 5.2 for LLaMA-3.
The perplexity results for OPT models on the C4 \cite{C4} and WikiText2 \cite{wikitext-2} datasets are presented in Table~\ref{opt-result}. Even with per-channel quantization, our approach outperforms previous group-wise weight-only quantization methods. For example, for OPT-125M, our per-channel ternary method maintains an average PPL of 39.92, while the 2-bit OmniQuant achieves only 75.53 with a group size of 128.

%c4,
\subsection{Results on Zero-Shot Tasks}
In this section, we assess the performance of ternarized models on zero-shot tasks and present the results in Table~\ref{llama-3-result} and Table~\ref{llama-2-result}. Previous work \cite{llama-3-quant} indicates that LLaMA-3 experiences significant degradation at ultra-low bit-widths, highlighting the substantial performance gap that needs to be addressed by the quantization research community. For LLaMA-3-7B, our approach (W1.58A16) outperforms the previous quantization-aware training method DB-LLM (W2A16) by 8.2\% in accuracy, improving the average accuracy from 51.8\% to 60.0\%.
Our methods also surpass previous best results in LLaMA-1 and LLaMA-2 by  1.56\% and 1.13\%, respectively, demonstrating superior preservation of generation capability in LLMs.
  
\subsection{Ablations}
In this section, we present some ablation studies to demonstrate the efficacy of our methods. For this segment, we limit the training to only 2,000 iterations.

\textbf{Dual Learnable Ternarization.}
We conducted experiments to investigate the effects of our DLT method. As shown in Table~\ref{Table_leranble}, the DLT method achieves a 1.49 PPL improvement on the OPT-1.3B model and a 0.89 PPL improvement on the LLaMA-1-7B model. DLT enables $\alpha$ and $\gamma$ to learn values better suited to LLM weights, enhancing the model's expressive capacity at extremely low-bit scenarios.

\begin{table}[h]
  \centering
  \begin{tabular}{lccc}
    \toprule
    \textbf{LLaMA-7B} \textbf{PPL} &\textbf{FP} &\textbf{TWN} & \textbf{DLT} \\
    \hline 
    OPT-1.3B &   14.63 &22.32       & \textbf{20.83}\\
    LLaMA-7B & 5.68 &10.10       & \textbf{9.21} \\
    \bottomrule
  \end{tabular}
  \caption{Comparison of various ternarization methods on OPT-1.3B and  LLaMA-7B model.}
  \label{Table_leranble}
\end{table}

\textbf{Outlier-Friendly Feature KD.}
We conducted experiments to evaluate various feature knowledge distillation methods and their combinations. Consistent with our analysis in Section~\ref{sec_difficulty}, mean square error (MSE) methods often suffer from training instability. We also explored other improved methods based on MSE, such as clamping outliers or disregarding abnormal loss. However, these methods only partially address the issue and involve significant hyperparameter tuning. As shown in Table ~\ref{Table_feature_kd}, Outlier-friendly Feature KD improves performance by 0.65 PPL and 0.57 PPL on Wikitext2 and C4, respectively. Additionally, combining it with logits, KD further enhances performance by 0.77 PPL.

\begin{figure}\CenterFloatBoxes
\begin{floatrow}
\ffigbox[\FBwidth]{%
  \includegraphics[width=0.45\textwidth]{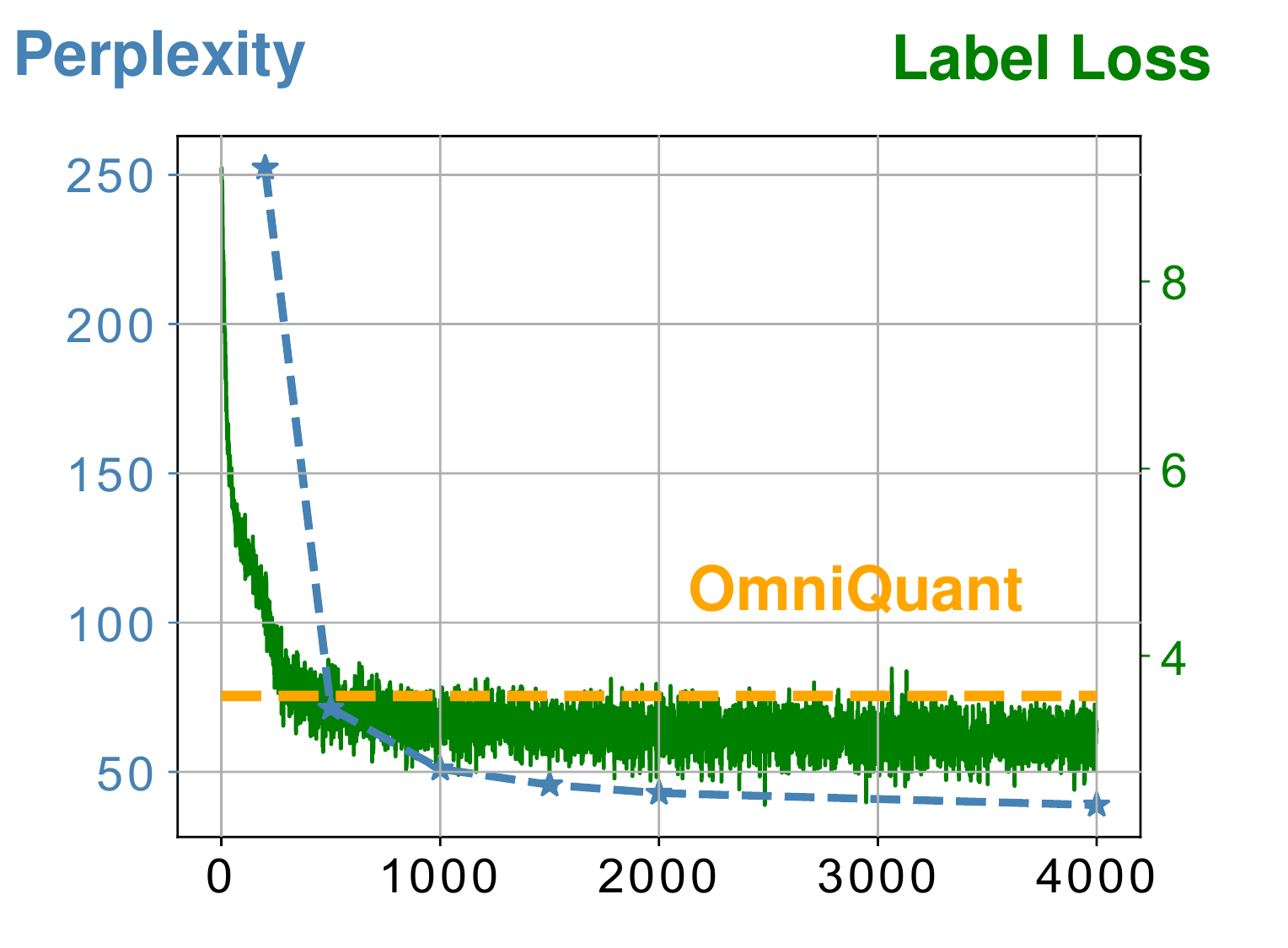}
}{%
  \caption{Training loss and validation perplexity curves. The experiments are conducted on OPT-125M with a group size of 128. Our method surpasses OmniQuant with only 500 steps. }%
  \label{fig_iteration}%
}
\killfloatstyle\ttabbox[\Xhsize]
{    \caption{
Comparison of different knowledge distillation techniques on the LLaMA-1-7B model. OFF and logits KD, either separately or combined, can improve performance.}
  \label{Table_feature_kd}}
{     \begin{tabular}{lcc}
    \toprule
   \textbf{Method} & \textbf{WikiText2} & \textbf{C4} \\
    \hline
    W/O KD & 9.98 &11.56 \\ 
    Logits  &   9.84  & 11.41\\  
    MSE &10.84 &  12.12  \\ 
    MSE + Clamp & 9.48 & 11.14\\
    MSE + Skip  &9.72& 11.39 \\
    OFF &9.33 &  10.99 \\ 
    \midrule
    Logits+OFF  & \textbf{9.21} & \textbf{10.78}  \\
    \bottomrule
  \end{tabular}
    }
\end{floatrow}
\end{figure}

% \begin{table}
%   \centering
%   \begin{tabular}{lll}
%     \hline
%     KD &WikiText2 & C4 \\
%     \hline
%     W/O KD & 9.98 &11.56 \\ 
%     Logits  &   9.84  & 11.41
%  \\  
%     MSE &10.84 &  12.12  \\ 
%     \hspace{0.25cm}+Clamp & 9.48 & 11.14\\
%     \hspace{0.25cm}+Skip  &9.72& 11.39 \\
%     OFF &9.33 &  10.99 \\ 
%     \hline
%     Logits+OFF  &9.21 &10.78  \\
%     \hline
%   \end{tabular}
%   \caption{Experiments were conducted to evaluate various knowledge distillation (KD) methods on the LLaMA-7B model }
%   \label{Table_feature_kd}
% \end{table}

% % \begin{figure}[hb]
% % \vskip 0.2in
% % \captionsetup[subfigure]{labelformat=empty} 
% \begin{center}
% % \centerline{\includegraphics[width=\columnwidth]{figures/asym_weight.pdf}}\captionsetup[subfigure]{labelformat=empty}
% \centering
% \includegraphics[width=0.40\textwidth]{figures/iterartion.pdf}
% \caption{Training loss and validation perplexity curves. }
% \label{fig_iteration}
% \end{center}
% \vskip -0.2in
% % \end{figure}

\textbf{Training Steps.} We conducted experiments on OPT-125M with a group size of 128 to observe the validation perplexity and training loss over steps. Our method demonstrates fast convergence, recovering performance in a few steps, and surpasses the previous post-training quantization method in just 500 steps.

% \begin{figure}[hb]
% \vskip 0.2in
% % \captionsetup[subfigure]{labelformat=empty} 
% \begin{center}
% % \centerline{\includegraphics[width=\columnwidth]{figures/asym_weight.pdf}}\captionsetup[subfigure]{labelformat=empty}
% \centering
% \includegraphics[width=0.40\textwidth]{figures/iterartion.pdf}
% \caption{Training loss and validation perplexity curves. }
% \label{fig_iteration}
% \end{center}
% \vskip -0.2in
% \end{figure}

\section{Conclusion}
In this paper, we introduce TernaryLLM, tailored for precise ternarized Large Language Models. To tackle the decline in accuracy, we propose Dual Learnable Ternarization, which integrates learnable scales and shifts to adapt to the asymmetric weight distribution in group-wise quantization. Additionally, we present Outlier-Friendly Feature Knowledge Distillation (OFF), designed to maximize the mutual information between floating-point and quantized models. This approach also takes advantage of the insensitivity of cosine similarity to outliers, thereby mitigating training instability.
%Additionally, we introduce Outlier-Friendly Feature Knowledge Distillation (OFF), leveraging cosine similarity as the distance metric to recover lost information in ternarized models. 
TernaryLLM not only enhances the compression ratio of parameters but also enables LLM inference with floating-point additions instead of multiplications.
Extensive experiments demonstrate that TernaryLLM outperforms previous extreme low-bit quantization methods on established NLP benchmarks. 
\clearpage
\bibliographystyle{unsrt}  
\bibliography{main}

\clearpage

\appendix
\section{Appendix / Supplemental Material}

\begin{center}
\begin{table*}[ht]
\centering
\caption{Training details for different models. We use the same training parameters for LLaMA-1, LLaMA-2 and LLaMA-3. }
\label{table_traing_details}
  \begin{tabular}{lllcccl}
    \toprule
    \textbf{Model} & \textbf{OPT-125M} & \textbf{OPT-1.3B} & \textbf{OPT-2.7B} & \textbf{OPT-6.7B} & \textbf{LLaMA-7B}\\
    \midrule
    Learning rate  &1e-4 &5e-5&1e-4&5e-5&1e-4 \\
    Logits KD loss coefficient & 0.001 &0.001  & 0.001& 0.001&  0.001\\
    Feature KD loss coefficient & 10 &10 & 10 & 10& 5 \\
    Feature KD number layers & 6 &  18&  18 &18 & 18 \\
    \bottomrule
  \end{tabular}
\end{table*}
\end{center}

\subsection{Training Details}

During the fine-tuning process, we optimize the trainable parameters using the AdamW optimizer with zero weight decay. The batch size is set to 16. We implement cosine learning rate decay, gradually decreasing to 0.1 times the initial value. All models undergo training for 10,000 steps, including a 500-step warm-up period. The training is conducted on a single 8-GPU node of AMD INSTINCT MI250. The learning rate for scales and shifts in DLT is set to 0.1 times weight learning rate. Additional parameters are provided in Table~\ref{table_traing_details}.

% \subsection{Results on PTB}
% The results for LLaMA-1-7B are shown in Table~\ref{llama-1-result}.

\subsection{Proof of Theorem 1}
\setcounter{theorem}{0}
\begin{theorem}
\label{theorem1}

Assume \(x \in \mathbb{R}^{C_i} \sim \mathcal{N}(0, \sigma_x^2)\), \(W = (w_1^T, w_2^T, \ldots, w_{C_o})^T\) and \(y = \text{RMSNorm}(Wx)\). Let \(W_q\) denotes the ternarization of \(W\) and \(y_q = \text{RMSNorm} (W_q x)\). The objective to maximize the mutual information between \(y\) and \(y_q\) (formally, \(\max I(y, y_q)\)) can been achieved by $\mathbb{E}_{p(x)}$cos(${W_qx}, Wx)=1 $.
\end{theorem}

\begin{proof}
From the definition of RMSNorm, we have:
\begin{equation}
    \begin{aligned}
        y= a \odot \frac{Wx}{||Wx||} + b
    \end{aligned}
\end{equation}

\begin{equation}
    \begin{aligned}
        y_q= a \odot \frac{W_qx}{||W_qx||} + b
    \end{aligned}
\end{equation}

Set $z=\frac{Wx}{||Wx||}$ and $z_q=\frac{W_qx}{||W_qx||}$, then we can simplify cosine similarity as follows:
\begin{equation}
\begin{aligned}
         \mathbb{E}_{p_(x)}\quad ||y-y_q||_2^2 
     & \leq \mathbb{E}_{p_(x)}\quad  || a \odot (z-z_q )||_2^2  \\ 
    &= \mathbb{E}_{p_(x)} \quad  ||a||_2^2 \cdot ||z- z_q||_2^2 \\ 
    &= \mathbb{E}_{p_(x)}\quad ||a||_2^2 \cdot  (||z||_2^2 +||z_q||_2^ 2-2||z||_2 ||z_q||_2 cos(y,y_q) ) \\
    &=  ||a||_2^2 \cdot (2 - \mathbb{E}_{p_(x)} {cos(y,y_q) }) = 0
\end{aligned}
\end{equation}
Thus, we have $\mathbb{E}_{p(x)} \{(y-y^q) (y-y^q)^T\}_{ij} \leq \mathbb{E}_{p_(x)}||y-y_q||_2^2 = 0$. Then, from the definition of mutual information, 
\begin{equation}
    \begin{aligned}
        I(y, y_q)=H(y) - H(y|y_q) \\
        & =  H(y) - H(y - y_q | y_q ) \\ 
         & \geq  H(y) - H(y - y_q) \\
       & \geq   H(y) - \frac{1}{2} \log(2\pi e)^n  \text{det} (\mathbb{E}_{p(x)} {(y-y^q) (y-y^q)^T}) \\
        & = H(y) 
    \end{aligned}
\end{equation}
The first inequality arises from the condition entropy $H(X|Y) \leq H(X)$, while the second inequality stems from Theorem~\ref{theorem2}.
\end{proof}

\begin{theorem}
\label{theorem2}
    Let the random vector $x \in \mathbb{R}^n$ have zero mean and covariance $K = \mathbb{E}[xx^T]$ (i.e., $K_{ij} = \mathbb{E}[X_iX_j], 1 \leq i, j \leq n$). Then $h(x) \leq \frac{1}{2} \log(2\pi e)n|\mathbf{K}|$, with equality iff $x \sim \mathcal{N}(0,\mathbf{K})$.
\end{theorem}
\begin{proof}
     The proof of this theorem can be found in information theory textbooks \cite{cover1999elements}.
\end{proof}

\subsection{Limitation}
Although TernaryLLM enhances parameter compression ratios and facilitates LLM inference through floating-point additions instead of multiplications, maximizing acceleration in LLM inference via ternarization requires greater hardware support. Previous research has succeeded in hardware accelerator for ternarized convolution networks \cite{eetha2021tilenet, zhu2022tab}. We hope this paper will attract more researchers to focus on customized hardware for ternarized LLMs.

\end{document}